\documentclass[letterpaper, 10pt, conference]{ieeeconf}  

\IEEEoverridecommandlockouts                              

\usepackage[T1]{fontenc}
\usepackage[utf8]{inputenc}
\usepackage{amsthm}
\usepackage{times}
\usepackage{multicol}
\usepackage[bookmarks=true]{hyperref}
\usepackage{xcolor}
\usepackage{hyperref}
\usepackage{amsmath, amssymb}
\usepackage{amsfonts}
\usepackage{graphicx}
\usepackage{siunitx}
\usepackage{standalone}
\usepackage{booktabs}
\usepackage[ruled,vlined,linesnumbered,noend]{algorithm2e}
\usepackage{mdframed}
\usepackage{fancyvrb,multirow}
\usepackage{soul}
\usepackage{dsfont,mathabx}
\usepackage{array, booktabs}
\usepackage{subfigure}
\usepackage{booktabs}
\usepackage{makecell}
\usepackage{threeparttable}

\newtheorem{theorem}{Theorem}

\theoremstyle{definition}

\newtheorem{definition}{Definition}
\newtheorem{remark}{Remark}

\title{\LARGE \bf Multi-UAV Deployment in Obstacle-Cluttered Environments \\with LOS Connectivity}

\author{Yuda Chen$^1$, Shuaikang Wang$^1$, Jie Li$^2$ and Meng Guo$^1$
  \thanks{
    The authors are with
	$^1$the Department of Mechanics and Engineering Science,
	College of Engineering, Peking University, Beijing 100871, China;
	and $^2$National University of Defense Technology, Hunan 410073, China.
    This work was supported by the National Natural Science Foundation
    of China (NSFC) under grants 62203017, T2121002, U2241214;
	2030-Key Project under Grant 2020AAA0108200.
    Corresponding author: Meng Guo, {\tt\small meng.guo@pku.edu.cn}.}
}

\begin{document}
\maketitle
\thispagestyle{empty}
\pagestyle{empty}

\setlength{\textfloatsep}{6pt}  


\begin{abstract}
    A reliable communication network is essential for multiple UAVs operating within obstacle-cluttered environments,
	where limited communication due to obstructions often occurs.
	A common solution is to deploy intermediate UAVs to relay information via a multi-hop network,
	which introduces two challenges:
	(i) how to design the structure of multi-hop networks;
	and (ii) how to maintain connectivity during collaborative motion.
	To this end, this work first proposes an efficient constrained search method based on the minimum-edge RRT$^\star$ algorithm, to find a spanning-tree topology that requires a less number of UAVs for the deployment task.
	Then, to achieve this deployment,
	a distributed model predictive control strategy is proposed for the online motion coordination.
	It explicitly incorporates not only the inter-UAV and UAV-obstacle distance constraints,
	but also the line-of-sight (LOS) connectivity constraint.
	These constraints are well-known to be nonlinear and often tackled by various approximations.
	In contrast, this work provides a theoretical guarantee
	that
	all agent trajectories are ensured to be collision-free with a team-wise LOS connectivity at all time.
	Numerous simulations are performed in 3D valley-like environments,
	while hardware experiments validate its dynamic adaptation
	when the deployment position changes online.
\end{abstract}

\section{Introduction}\label{sec:intro}
Fleets of unmanned aerial vehicles (UAVs) have been widely deployed to accomplish
collaborative missions for exploration, inspection and rescue~\cite{Chung2018,Quan2022,Hu2024,tian2024ihero}.
In many applications, the UAVs need to keep a reliable communication with a ground station to transfer data and monitor status of the fleet~\cite{Yin2023}.
This however can be challenging due to the limited communication range and obstructions from obstacles.
An effective solution as also adopted
in~\cite{Yin2023,zhang2025fly} is to deploy more UAVs as relay nodes to form a multi-hop
communication network.
Two key challenges arise with this approach, i.e.,
(i) how the multi-hop network should be designed,
including the topology and the positions of each node;
(ii) how to control the UAV fleet to form this network,
while maintaining the team-wise connectivity and avoiding collisions
in an obstacle-cluttered environment.

\begin{figure}[t]
  \centering
  \includegraphics[width=1\linewidth]{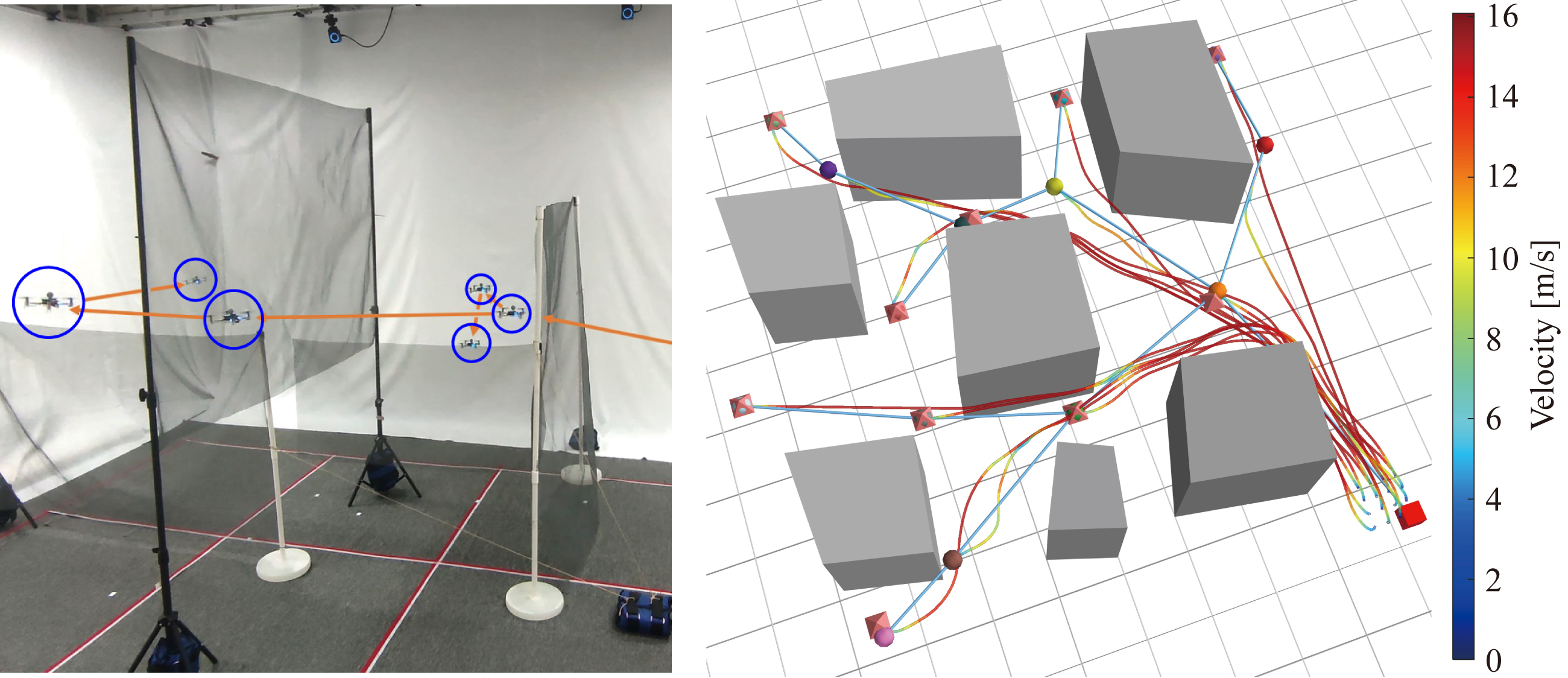}
  \vspace{-5mm}
  \caption{
  \textbf{Left:} Experiment of $6$ quadrotors cooperatively searching two targets.
  \textbf{Right:} Simulation of $15$ agents (colored spheres) are deployed from the ground station (red cubic) to reach $10$ targets (pink diamonds).
		The optimized network topology (blue lines) are kept during motion at all time.}
  \label{fig:overall}
\end{figure}

\subsection{Related Work}\label{subsec:intro-related}
Regarding the design of communication network,
linear graphs are proposed in~\cite{Schouwenaars2006,Varadharajan2020}
for exploration and surveillance tasks in complex environments.
Such structure is simple to design but rather limited, as it only allows the tailing agent to perform a task.
Another common topology is called spanning tree as adopted in~\cite{Majcherczyk2018,Luo2020,Yi2021},
where all agents as leaves of this tree can execute multiple tasks simultaneously.
However, these works often assume a redundant number of agents
without estimating or minimizing the number of agents that is required to accomplish the tasks.
Furthermore, the communication model considered in this work is often simplified to be range-based,
i.e., any two agents within a certain distance can exchange information.
However, obstructions due to obstacles can severely degrade the communication quality and the transmission rate.
Thus, as also motivated in~\cite{Esposito2006,Boldrer2021,guo2018multirobot},
the line-of-sight (LOS) model is a more practical choice for approximating the inter-agent communication in obstacle-cluttered environments.
More concretely, the work in~\cite{Boldrer2021} proposes the notion of visibility-weighted minimum spanning tree (MST),
which leads to a more robust path compared to~\cite{Luo2020} under LOS constraints.
Lastly, more realistic communication models for multi-agent system are
proposed in~\cite{Yin2023, Clark2C022,tian2024ihero},
which take into account reflection and diffraction effects with a higher-order model.
Whereas essential for certain applications,
such models are often highly nonlinear and non-analytical,
yielding them difficult to be incorporated
into the collaborative motion strategy.
Thus, despite of being a conservative approximation,
the considered LOS constraint can ensure a high quality of inter-agent communication.
\vspace{-1mm}

On the other hand,
it is a non-trivial problem to maintain the designed communication network during collaborative motion,
especially when avoiding collisions between agents and obstacles.
Numerous methods have been proposed to address these constraints.
For instances, artificial forces are designed in~\cite{Feng2015,Paolo2013},
e.g., repulsive forces for obstacles and other agents;
and attractive forces for the targets and connected agents.
Lyapunov-based nonlinear controllers are proposed in~\cite{zavlanos2011graph,guo2013controlling}
mostly for free-spaces.
Control barrier certificates are proposed in~\cite{Luo2020,Wang2016}
for both  constraints of connectivity maintenance and collision avoidance.
Despite of their intuitiveness,
these methods often lack theoretical guarantee on the safety when encountering control and state constraints.
Furthermore, the LOS restriction is particularly difficult to incorporate due to its non-convexity.
A geometric approach is proposed in~\cite{Boldrer2021} based on the Loyd algorithm,
which drives the agents towards the intersection of neighboring communication zones.
It however neglects the dynamic constraints such as limited velocity and acceleration.
Optimization-based methods such as~\cite{Caregnato2022} directly formulate these constraints
as mixed integer linear programs (MILP),
of which the complexity increases drastically with the number
of agents and obstacles, yielding them impractical for real-time applications.

\subsection{Our Method}\label{subsec:intro-our}
To tackle these limitations, this work proposes an integrated framework of network formulation
and motion coordination for deploying UAV fleets in obstacle-cluttered environments.
It generates an embedded network topology as spanning trees via the minimum-edge RRT$^\star$ algorithm,
subject to the LOS restrictions.
During online execution, a collaborative motion coordination strategy is proposed based on the
distributed model predictive control (MPC) that maintains the designed network topology
and generates collision-free trajectories.
In particular, the LOS constraints as well as the requirement of collision avoidance are explicitly formulated as linear constraints by restricting the UAVs in constructed convex polyhedra.
It has been shown that the MPC can be solved online efficiently,
with a theoretical guarantee on its feasibility.
Effectiveness of the proposed framework is validated in numerous simulations and hardware experiments.
The main contributions are two-fold:
\begin{itemize}
	\item[$\bullet$]
	A planning method is introduced to construct a spanning-tree-like communication topology
	with a minimum number of edges.
	Different from the common self-organizing topology like \cite{Varadharajan2020,Luo2020,Boldrer2021},
	the proposes algorithm of topology designing decreases the number of agents required for the task.

	\item[$\bullet$] The proposed motion control strategy based on
	distributed MPCs theoretically guarantees the network connectivity and collision-avoidance.
	Compared with MILP-based methods in \cite{Caregnato2022} which handles LOS restrictions,
	our method reduces the planning time by $90\%$ and exhibits online adaptation
	towards dynamic targets.
\end{itemize}

\section{Problem Description}\label{sec:problem}

\subsection{Robot Model}\label{subsec:robots}
As shown in Fig.~\ref{fig:overall},
consider a fleet of $N>0$ UAVs operating within the 3D workspace~$\mathcal{W}\subset \mathbb{R}^{3}$.
Each agent follows the standard double-integrator model
$\dot{x}^i(t) \triangleq [v^i(t), u^i(t)]$,
where $x^i(t) \triangleq [p^i(t),\, v^i(t)]$ is the state of agent~$i$;
$p^i(t), v^i(t), u^i(t) \in \mathbb{R}^3$ are the position, velocity and control input of agent~$i$ at time~$t\geq 0$,
respectively, $\forall i\in \mathcal{N} \triangleq  \{1,\cdots,N\}$.
Moreover, the state and control input of each agent~$i\in \mathcal{N}$
are restricted by:
\begin{equation*}
	\| \Theta_a u^i(t) \|_2 \le a_{\text{max}}, \; \| \Theta_v v^i(t) \|_2 \le v_{\text{max}},
\end{equation*}
where~$\Theta_v,\, \Theta_a$ are positive-definite matrices,
and $v_{\text{max}},\, a_{\text{max}}>0$ are the maximum velocity and acceleration, respectively.
Note that the total number of agents $N$ is \emph{not} fixed,
rather designed as part of this problem.

\subsection{Communication Constraints}\label{subsec:communication}
Any pair of agents~$i,j\in \mathcal{N}$ can direct communicate
as \emph{neighbors} if both the following two conditions hold:
(i) their relative distance is within the communication range~$d_c>0$,
i.e., $\|p^i(t)-p^j(t)\|_2 \leq d_c$;
(ii) the line-of-sight (LOS) between the agents
cannot be obstructed by obstacles, i.e.,
$\textbf{Line}(p^i(t),\, p^j(t))\, \cap \, \mathcal{O} = \emptyset$,
where~$\textbf{Line}(p^i(t),\, p^j(t))$ is the line segment
from position~$p^i(t)$ to position~$p^j(t)$;
and
$\mathcal{O}\subset \mathcal{W}$ is the volume occupied
by a set of convex-shaped static obstacles,
each of which is defined as the convex hull of a set of known vertices.
Lastly, there is a ground station at position~$p_g \in \mathcal{W}$,
which needs to communicate with any agent as described above.
With a slight abuse of notation,
the ground station is denoted by agent~$0$ and $\widetilde{\mathcal{N}} \triangleq \{ 0, 1, \cdots, N \}$.

Consequently, an undirected and time-varying
communication network
$G(t)\triangleq ( \widetilde{\mathcal{N}} ,\, E(t))$
can be constructed given the neighboring relation above,
where
$E(t) \subset \widetilde{\mathcal{N}} \times \widetilde{\mathcal{N}}$
for time~$t>0$.
In other words, $(i,\,j)\in E(t)$ if the position of agents~$i$ and $j$
at time~$t$ satisfies the two conditions above.
The network $G(t)$ is called \emph{connected} at time~$t$
if there exists a path between any two nodes in $G(t)$, i.e.,
$G(t) \in \mathcal{C}$, where $\mathcal{C}$ denotes the set of all connected graphs

\subsection{Collision Avoidance}\label{subsec:collision-avoid}
Each agent~$i\in \mathcal{N}$
occupies a spherical region with radius~$r_a>0$, denoted by
$\mathcal{B}^i(t)\triangleq \{p\in \mathcal{W}\,|\, \|p-p^i(t)\|_2 \leq r_a\}$.
To avoid collisions, the spherical regions of any two agents~$i,j\in \mathcal{N}$
must satisfy the condition~$\mathcal{B}^i(t)\cap \mathcal{B}^j(t) = \emptyset$.
Furthermore, each agent~$i\in\mathcal{N}$ must avoid colliding with static obstacles at time~$t>0$,
meaning that the condition $\mathcal{B}^i(t) \cap \mathcal{O} = \emptyset$ must be satisfied.

\subsection{Problem Statement}\label{subsec:prob-statement}
Lastly, there are $M>0$ target positions within the workspace~$\mathcal{W}$,
denoted by
$\mathcal{P}_{\text{tg}} \triangleq  \left\{p_{{\text{tg}},m}, \, m \in \mathcal{M} \right\}\subset \mathcal{W}\backslash \mathcal{O}$,
where $\mathcal{M}\triangleq \{1,\cdots,M\}$.
Each target position is associated with potential tasks to be performed by the agents.
Each agent~$i\in \mathcal{N}$ starts from its initial position~$p^i_0$,
which is collision-free and located in close proximity to the ground station.
Additionally, the initial communication network $G(0)$ is assumed to be connected.
Thus, the problem is formulated as a constrained optimization problem:
\begin{equation}\label{eq:problem}
	\begin{split}
		&\textbf{min}_{\{u^i(t),\, T\}}\; N\\
		\textbf{s.t.}\quad
		& \dot{x}^i(t) = [v^i(t),u^i(t)],\,\forall i,\,\forall t; \\
		& \| \Theta_a u^i(t) \|_2 \le a_{\text{max}}, \; \| \Theta_v v^i(t) \|_2 \le v_{\text{max}},\,\forall i,\,\forall t;\\
		& \mathcal{B}^i(t)\cap \mathcal{B}^j(t) = \emptyset, \; \mathcal{B}^i(t)\cap \mathcal{O} = \emptyset,\,\forall i,\,\forall j,\,\forall t;\\
        & p^{i_m}(T)=p_{{\text{tg}},m},\,\forall m;\, G(t) \in \mathcal{C},\,\forall t;
	\end{split}
\end{equation}
where~$t\in [0,\, T]$;
$T>0$ is the termination time when all target positions are reached;
$i_m\in \mathcal{N}$ is the agent that reaches target~$m\in \mathcal{M}$;
the objective is to minimize the total number of agents required for the mission;
and the constraints are the aforementioned dynamic model, communication connectivity,
and safety guarantees.
\section{Proposed Solution}\label{sec:solution}
The proposed solution consists of two layers.
First, the communication network is formulated given the
workspace and target locations to minimize the team size.
Second, a collaborative control strategy is designed to drive the agents
towards the target positions while maintaining the desired topology and avoiding collisions.

\subsection{Design of the Communication Network}\label{sec:network-formulation}

\subsubsection{Minimum Edge RRT$^*$}\label{subsubsec:RRT}
To begin with, a modification of the informed RRT$^\star$,
as introduced in \cite{Gammell2014},
is proposed for the design of the communication network..
The proposed method, referred to as \texttt{MiniEdgeRRT$^\star$},
aims to find a set of paths from a starting point to a set of target points,
minimizing the number of edges along each path.
Specifically, the following three modifications are implemented.

\emph{Goal Region:} The goal region is defined as the union of vertices surrounding the target points,
    allowing for the concurrent identification of paths toward multiple targets.

\emph{Cost:} The cost between neighboring vertices in the tree is
	$\textbf{c}(\widehat{\nu},\,\nu) \triangleq \|\widehat{\nu}-\nu\|_2 + \kappa$,
	where~$\widehat{\nu}$ represents the parent node of~$\nu$ and~$\kappa>0$ is a large penalty
	added to the cost of each edge.
	The total cost of a path~$\Gamma$ is the accumulated cost of all edges alone the path,
	denoted by~$\texttt{Cost}(\Gamma)$.
	Thus, paths with fewer edges are always preferred.

\emph{Sample:} New samples~$\nu_{\text{new}}$ are generated within a distance~$d_c$ of
	an existing sample, i.e., $\|\nu_{\text{new}}-\nu_{\text{nearest}}\|\leq d_c$,
	ensuring that any two neighboring vertices satisfy the communication-range constraint.

%
%

\begin{remark}
	Based on the analysis in~\cite{Gammell2014},
    it can be shown that a path with fewer edges can be found
    for each target point in~$\mathcal{P}_{{\text{tg}}}$ given the starting vertex.
	\hfill $\blacksquare$
\end{remark}


\subsubsection{Network Formulation}\label{subsubsec:network-formula}
The paths planned by \texttt{MiniEdgeRRT$^\star$} are initially treated as independent paths to the targets,
which can be further improved by increasing the number of \emph{shared} edges.
First of all, a spanning-tree topology is defined as follows.


\begin{definition}
	An \emph{embedded} spanning tree is defined as a set of 3-tuples:
	$\mathcal{T} \triangleq \{\nu_0, \nu_1, \cdots \}$,
	where each vertex~$\nu_i \triangleq  (i,\, p_i,\, I_i)$,
	with index~$i\in \mathbb{Z}$,
	position~$p_i\in \mathcal{W}\backslash \mathcal{O}$
	and parent index~$I_i\in \mathbb{Z}$ of vertex~$\nu_i$.
	Note that~$\nu_0 \triangleq (0,\, p_{0},\,\emptyset)$ is
	the root, corresponding to the ground station. \hfill $\blacksquare$
\end{definition}

As summarized in Alg.~\ref{AL:span-tree},
an iterative algorithm is proposed
to incrementally construct the spanning tree
by greedily adding target points
based on their minimum-cost path.
Specifically, the spanning tree $\mathcal{T}$ is initialized
with the newly-added vertices~$\mathcal{V}_\text{new}$,
the set of indices $\mathcal{I}$ associated with the target points,
and the current-best path $\Gamma_c^i$ for each target~$i \in \mathcal{I}$,
as shown in Line~\ref{algline:initial-path_tree}.
During each iteration,
the set of paths~$\widehat{\Gamma}$ from each target~$i\in \mathcal{I}$ to every vertices in $\mathcal{V}_\text{new}$
is derived via the proposed \texttt{MiniEdgeRRT}$^\star$ in Line~\ref{algline:path-list}.
The path with the minimum cost~$\Gamma^\star$ is selected from the $\widehat{\Gamma}$ in Line~\ref{algline:get-best-path}.
Subsequently, the current-best path~$\Gamma_c^i$ for target~$i$ is updated by
comparing it with~$\Gamma^\star$ in Line~\ref{algline:path-update}.
The index of the target point with the minimum cost~$i^\star\triangleq \text{argmin}_i\{\Gamma_t^{i^\star}\}$
is determined in Line~\ref{algline:best-index}.
Then, the associated target~$i^\star$ is removed from the set of potential targets~$\mathcal{I}$,
and the nodes extracted from~$\Gamma_t^{i^\star}$ are added to
the spanning tree as~$\mathcal{V}_\text{new}$ in Lines~\ref{algline:new-point}
and~\ref{algline:add-branch}, with their preceding nodes as parents.
This process is repeated until the set~$\mathcal{I}$ is empty,
indicating that all targets have been connected to the tree~$\mathcal{T}$.
As illustrated in Fig.~\ref{tree-span},
the number of edges in this topology is significantly fewer
than in a standard minimum spanning tree.

\begin{figure}[t]
	\centering
	\includegraphics[width=1\linewidth]{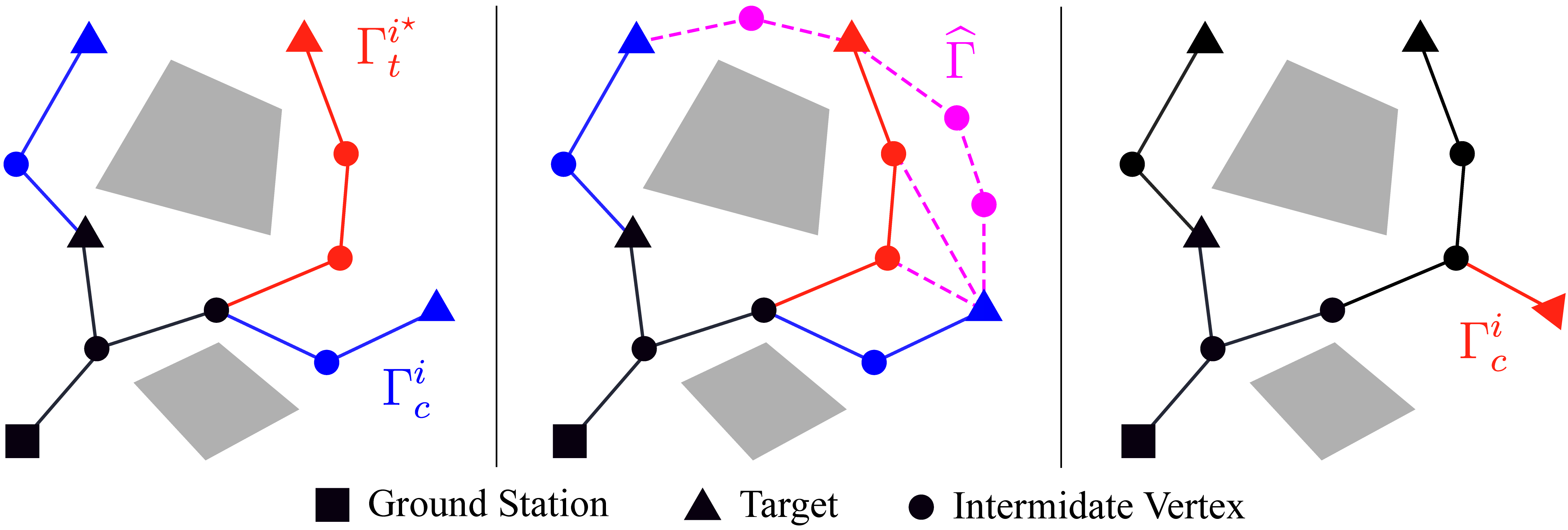}
    \vspace{-5mm}
	\caption{
			Illustration of the iterative process in Alg.~\ref{AL:span-tree}.
            \textbf{Left}: The target with the shortest path $\Gamma^{i^\star}_t$ (in red) is added to the tree,
            while the remaining targets are shown with their current-best paths~$\Gamma^{i}_c$ (in blue).
            \textbf{Middle}: Minimum-edge paths from other targets to the newly added points
            $\mathcal{V}_\text{new}$ are generated (in pink).
            \textbf{Right}: The updated topology after a new branch (in red) is added.
		}
	\label{tree-span}
\end{figure}

\begin{algorithm}[t!] \label{AL:span-tree}
	\caption{$\texttt{ComNet}()$}
	\SetKwInOut{Input}{Input}
	\SetKwInOut{Output}{Output}
	\Input{$p_0$, $\mathcal{P}_{\text{tg}}$, $\mathcal{O}$}
	\Output{$\{\Gamma^i\}$}
	$\mathcal{T} \leftarrow \{ \nu_0 \}$;
	$\mathcal{V}_\text{new} \leftarrow \left\{ \nu_0 \right\}$;
	$\mathcal{I}\leftarrow \mathcal{M}$;
	$\Gamma_c^i \leftarrow \emptyset, \; \forall i \in \mathcal{I}$\; \label{algline:initial-path_tree}
	\While{$\mathcal{I} \neq \emptyset$}{
		\For{ $i$ \rm{\textbf{in}}  $\mathcal{I}$ parallel \label{algline:span} }{
			$\widehat{\Gamma} \leftarrow \texttt{MiniEdgeRRT}^\star(p_{{\text{tg}},i},\, \mathcal{V}_\text{new},\, \mathcal{O})$ \label{algline:path-list}\;
			$\Gamma^\star \leftarrow \text{argmin}_{\Gamma\in \widehat{\Gamma}}\, \{\texttt{Cost}(\Gamma)\}$ \label{algline:get-best-path}\;
			$\Gamma_c^i \leftarrow \texttt{Compare}(\Gamma^\star,\, \Gamma_c^i)$ \label{algline:path-update} \;
		}
		$i^\star \leftarrow \text{argmin}_{i\in \mathcal{I}}\,\{\texttt{Cost}(\Gamma_c^i)\}$ \label{algline:best-index} \;
		$\mathcal{I} \leftarrow \mathcal{I} \backslash i^\star$ \label{algline:remove} \;
		$\mathcal{V}_\text{new} \leftarrow$ Extract vertices from $\Gamma_t^{i^\star}$ \label{algline:new-point} \;
		$\mathcal{T} \leftarrow \mathcal{T} \cup \mathcal{V}_\text{new}$ \label{algline:add-branch} \;
	}
	Compute the final $\Gamma^i$, $\forall i \in \mathcal{I}$ within $\mathcal{T}$ \label{algline:get-path} \;
\end{algorithm}

The computational complexity of Alg.~\ref{AL:span-tree} is similar to that of the
informed RRT$^{\star}$ in~\cite{Gammell2014},
i.e., $n_\text{sample} \log(n_\text{sample})$ where $n_\text{sample}$ is the number of sampled vertices.
In our implementation, the planning time for each call to the subroutine \texttt{MiniEdgeRRT}$^\star$
is limited by a given duration~$t_d>0$.
The computations for each target is performed in parallel,
so the overall complexity is bounded by~$\mathcal{O}\big(M n_\text{sample} \log(n_\text{sample})\big)$.
Note that the proposed algorithm does not generate the optimal spanning tree,
as this requires an exhaustive search over all possible target point sequences.

Given the spanning tree~$\mathcal{T}$ derived above,
the agents are divided into two
groups~$\mathcal{N}\triangleq \mathcal{N}_s \cup \mathcal{N}_c$.
The agents in~$\mathcal{N}_s$ are deployed at the target points
as \emph{searchers} for task execution,
while the agents in~$\mathcal{N}_c$ are assigned to the intermediate vertices as \emph{connectors}.
Thus, the total number of agents required to form the tree~$\mathcal{T}$ is given by
$N=|\mathcal{T}|-1$ with $N_s = M$ searchers
and~$N_c = N-M$ connectors.
Furthermore, the reference path for each searcher~$i\in \mathcal{N}_s$
is denoted by~$\Gamma^i\triangleq \{ p_g, \cdots, p_{{\text{tg}},n_i} \}$,
representing a sequence of waypoints along the
path in~$\mathcal{T}$ from the root
to the assigned target $p_{ {\text{tg}},n_i}$.
In contrast, the connectors do not have reference paths
and serve solely as intermediate relays.

\begin{remark}
		Related works propose self-organizing typologies such
		as line graphs in~\cite{Varadharajan2020} or spanning trees in~\cite{Luo2020,Boldrer2021}.
		Although they are applied to dynamic tasks in various environments,
		they do not consider the number of agents required as part of the optimization criteria.
	\hfill $\blacksquare$
\end{remark}


\subsection{Collaborative Trajectory Planning via Distributed MPC}\label{sec:MPC}
Given the desired communication topology~$\mathcal{T}$,
this section introduce a distributed control strategy to maintain this
topology~$\mathcal{T}$ during collaborative motion,
while ensuring collision avoidance.
The approach reformulates the problem with the distributed MPC framework~\cite{Franze2021},
where the geometric and dynamic constraints are approximated
and expressed as linear or quadratic inequalities.

To begin with,
denote by~$h >0$ the sampling time and~$K \in \mathbb{Z}^+$ the planning horizon.
The planned state and control input  of agent~$i$ at time $t+kh$ are denoted by
$x^i_k(t)\triangleq [p^i_k(t),\, v^i_k(t)]$ and $u^i_k(t)$, respectively,
$\forall k\in \mathcal{K} \triangleq  \{1,\cdots,K\}$.
Furthermore, the planned trajectory of
agent~$i$ at the time~$t$ is defined as
$\mathcal{P}^i\triangleq \{ p^i_1, p^i_2, \cdots, p^i_K \}$.
{For brevity, in the sequel, the index ``$(t)$" can be omitted if no ambiguity causes.}
Moreover, the so-called \emph{predetermined} trajectory is denoted by
$\overline{\mathcal{P}}^i \triangleq
\{ \overline{p}_{1}^{i}, \overline{p}_{2}^{i}, \cdots, \overline{p}_{K}^{i} \}$,
where $\overline{p}_{k}^{i}(t) \triangleq p^i_{k+1}(t-h)$,
$\forall k \in \{1,2,\cdots,K-1\}$ and $\overline{p}^i_{K}(t) \triangleq p^i_{K}(t-h)$.
Note that the planned trajectory~$\mathcal{P}^i$ should comply
with the double integrator model, i.e.,
\begin{equation} \label{dynamic-constraint}
	x_{k}^{i} \triangleq \mathbf{A} x_{k-1}^{i}+\mathbf{B} u_{k-1}^{i}, \forall k \in \mathcal{K},
\end{equation}
where
$\mathbf{A}
\triangleq
\left[
\begin{array}{ccc}
	\mathbf{I}_3 & h \mathbf{I}_3  \\
	\mathbf{0}_3 & \mathbf{I}_3 \\
\end{array}
\right]$,
$\mathbf{B}
\triangleq
\left[
\begin{array}{ccc}
	\frac{h^2}{2} \mathbf{I}_3   \\
	h \mathbf{I}_3
\end{array}
\right]$;
and the constraints for velocities and inputs remain
the same as follows:
\begin{equation} \label{eq:motion-constriant}
\begin{aligned}
	\| \Theta_a u^i_{k-1} \|_2  \le a_\text{max} , \,\,
	\| \Theta_v v_k^i \|_2  \le v_\text{max},   \forall k \in \mathcal{K}.
\end{aligned}
\end{equation}

\subsubsection{Collision Avoidance}
Collision avoidance among the agents can be reformulated via the
modified buffered Voronoi cell (MBVC) as proposed in our previous work~\cite{chen2024deadlock}, i.e.,
\begin{equation} \label{eq:inter-constraint}
	{a_{k}^{i j}}^{\mathrm{T}} p_{k}^{i} \geq b_{k}^{i j}, \ \forall j\neq i, \;  \forall k \in \mathcal{K},
\end{equation}
where the coefficients $a_{k}^{i j}$ and $b_{k}^{i j}$ are given by:
\begin{equation*}
\begin{aligned}
	& a_{k}^{i j}\triangleq \frac{ \overline{p}_{k}^{i}-\overline{p}_{k}^{j} } { \|\overline{p}_{k}^{i}-\overline{p}_{k}^{j}\|_{2} }, \,\,
	& b_{k}^{i j} \triangleq a_{k}^{i j^{\mathrm{T}}} \frac{\overline{p}_{k}^{i} + \overline{p}_{k}^{j}}{2}+\frac{r_{\min }^{\prime}}{2},
\end{aligned}
\end{equation*}
and $r^{\prime}_{\text{min}} \triangleq \sqrt{4{r_a}^{2}+h^{2} v_{\text{max} }^{2}} $.
On the other hand,
the collision avoidance between agent~$i\in \mathcal{N}$ and all obstacles
is realized by restricting its planned trajectory in a safe corridor formed by convex polyhedra.
This corridor serves as the boundary that separates the planned positions $p^i_k$
and the inflated obstacles given by
$\widehat{\mathcal{O}}(r_a) \triangleq \{p \in \mathcal{W} \ | \
\|p-p_o\|_2\leq r_a,\, p_o \in \mathcal{O}\}$.
Specifically, the constraint over the planned trajectory is stated as
the following linear inequality:
\begin{equation} \label{eq:obstacle-constraint}
	{a_{k}^{i,o}}^{\mathrm{T}} p_{k}^{i} \geq b_{k}^{i,o},\, \forall k \in \mathcal{K},
\end{equation}
where $a_{k}^{i,o}$ and $b_{k}^{i,o}$ represent the linear boundary for agent~$i$ at step $k$, w.r.t. obstacle $o \in \mathcal{I}_o$;
and $\mathcal{I}_o$ is the collection of obstacle indices.
The detailed derivation is omitted here due to limited space and refer
the readers to~\cite{Chen2022-2}.

Moreover,
as illustrated in Fig.~\ref{get_target_point},
an intermediate target point denoted by~$p_{\text{tg}}^i$ is computed dynamically
for each searcher~$i \in \mathcal{N}_s$.
Initially, $p_{\text{tg}}^i = p_g$.
For time $t>0$, $p_{\text{tg}}^i(t)$ is updated online as the closest vertex
in its path~$\Gamma^i$ that
\begin{equation*}
	\textbf{Conv}\big(  \left\{  \Gamma^i\left( p_{\text{tg}}^i(t-h), p_{\text{tg}}^i(t)\right),
	\,\overline{p}^i_K(t) \right\}  \big) \cap \widehat{\mathcal{O}}(r_a) = \emptyset,
\end{equation*}
where
$\Gamma^i( p_1, p_2 )$
denotes the sequence of nodes on $\Gamma^i$ from position $p_1$ to $p_2$;
and
$\textbf{Conv} (P) \triangleq \{\sum_{j=1}^J \theta_j p_j \,|\, \sum_j \theta_j =1,\,
p_j \in P,\, \theta_j \geq 0, \, j=1,\cdots,J\}$
is the convex hull formed by points within the set~$P$;
$p_{\text{tg}}^i(t-h)$ is the intermediate target at time~$t-h$.
For connector $i \in \mathcal{N}_c$,
its intermediate target point
is determined by $\overline{p}^i_K$.
Then, the position of the predetermined trajectory at the $(K+1)$-th step is
chosen as~$\overline{p}_{K+1}^{i} \triangleq  p^i_{{\text{tg}}}$.

\begin{figure}[t!]
	\centering
	\includegraphics[width=0.7\linewidth]{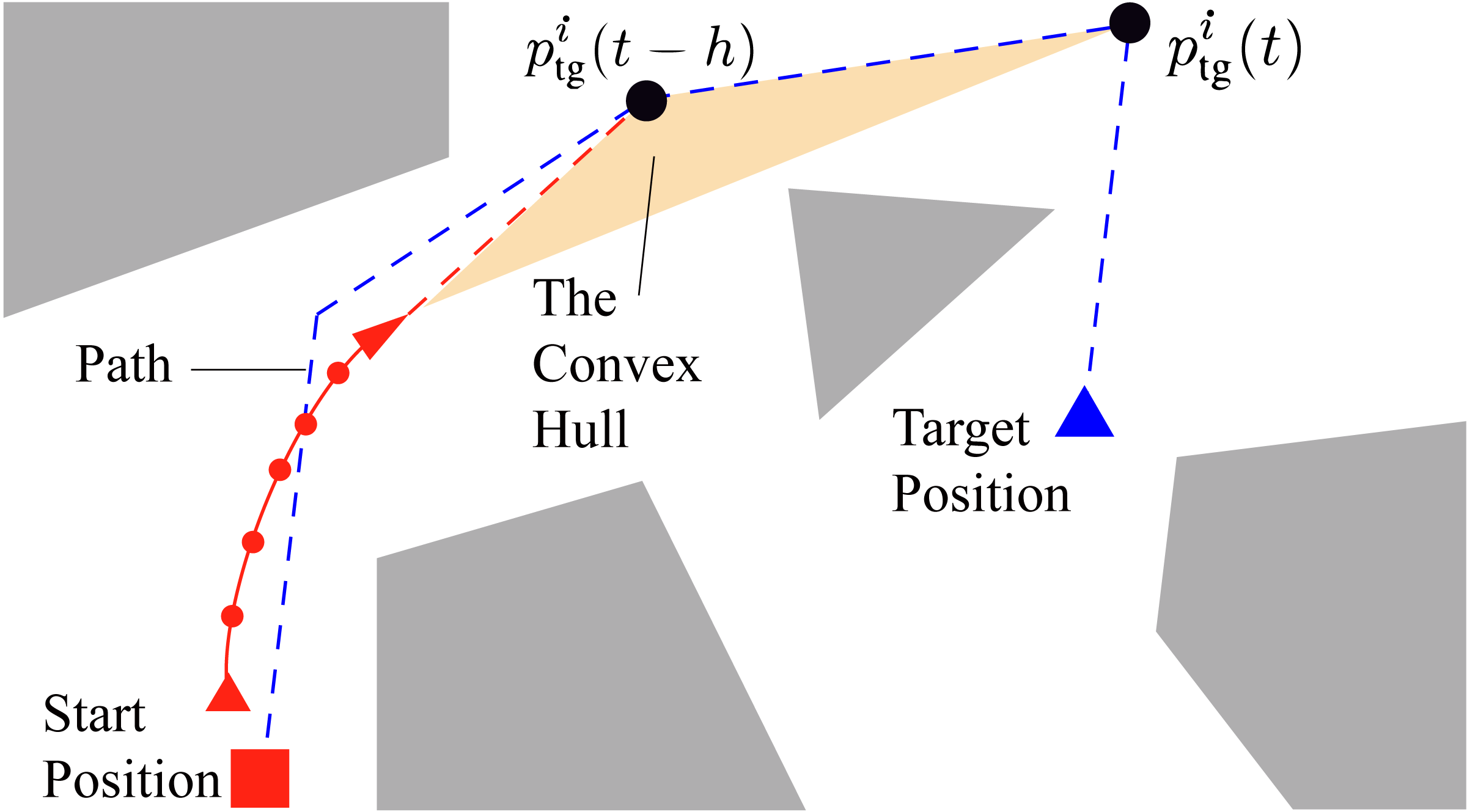}
    \vspace{-1mm}
	\caption{
		Illustration of how the intermediate target point~$p_{\text{tg}}^i(t)$
		is determined for a searcher agent~$i\in \mathcal{N}_s$ given its path~$\Gamma^i$.
		}
	\label{get_target_point}
\end{figure}

\subsubsection{Connectivity Maintenance} \label{sub: connection-maintain}
Maintaining the desired topology during motion is the most challenging constraint
to deal with.
For each pair of neighboring agents in~$\mathcal{T}$,
this constraint is decomposed into two sub-constraints:
(i) the relative distance between them should be less than~$d_{c}$;
(ii) the LOS connecting them should not be obstructed by obstacles.

The first sub-constraint is reformulated as a condition requiring
each agent stays within a sphere,
of which the center depends on its relative distance to the neighbor.
Namely,
the position of agent~$i$ at step $k$ is constrained by:
\begin{equation} \label{eq:circle-constraint}
	\left \| p_{k}^{i} - c^{i j}_k \right \| _2  \leq d_c/2,
	\, \forall j \in \mathcal{N}^i, \, \forall k \in \mathcal{K},
\end{equation}
where~$\mathcal{N}^i \subseteq \mathcal{N}$ is the set of parent and child
agents of agent~$i$ within the topology~$\mathcal{T}$;
and the center $c_k^{ij}\in \mathcal{W}$ is determined based on the predetermined trajectory of agents $i$ and $j$
by considering three different cases:
\begin{itemize}
	\item[(i)] If $ \| \overline{p}^i_k - \overline{p}^j_k \|_2 > d_w$
	with~$d_{w}>0$ being a chosen distance slightly less than~$d_{c}$,
	it indicates that agents~$i$ and~$j$ are far away and
	$c_k^{ij}= (\overline{p}^i_k + \overline{p}^j_k) / 2$;

	\item[(ii)] If $\|p-c^i_{k,\text{center}}\|<d_{w}$
	$\forall p\in \{\overline{p}^i_k, \overline{p}^j_k, \overline{p}^i_{k+1}, \overline{p}^j_{k+1}\}$,
	it indicates that agents~$i$ and $j$ are in close proximity and
	$c_k^{ij}=c^i_{k,\text{center}}$ with
	$c^i_{k,\text{center}} \triangleq {(\overline{p}^i_k+ \overline{p}^j_k+ \overline{p}^i_{k+1}+ \overline{p}^j_{k+1}) / 4}$;

	\item[(iii)] Otherwise, they are in medium distance and
	$c_k^{ij} =  \eta^{ij} (\overline{p}^i_k + \overline{p}^j_k) / 2
		+ (1-\eta^{ij}) c^i_{k,\text{center}}$,
	where~$\eta^{ij}\in [0,\,1]$ is solved by the following optimization problem:
	\begin{equation} \label{eq:eta}
		\begin{aligned}
			& \eta^{ij}= \min_{0\leq \eta \leq 1}\;  \eta, \\
			\textbf{s.t.} \;\; & \| \overline{p}^i_k - \overline{c}_k^{ij} \|_2 \leq  d_w / 2,
			\, \| \overline{p}^j_k - \overline{c}_k^{ij} \|_2 \leq  d_w / 2,\\
			& \overline{c}_k^{ij} = \eta (\overline{p}^i_k + \overline{p}^j_k) / 2 +  (1-\eta) c^i_{k,\text{center}},\\
		\end{aligned}
	\end{equation}
	which can be solved by the standard algorithm of bisection method \cite{Burden2015} with the search interval $[0,1]$.
\end{itemize}

Secondly, for the LOS sub-constraint,
linear constraints define a polyhedral safe zone for $p^i_k$ and $p^j_k$.
Before presenting these constraints, two intermediate points are computed:
$$
\overline{p}^i_{k,\text{next}} \triangleq \overline{p}^i_{k} + \xi^\star (\overline{p}^i_{k+1}-\overline{p}^i_{k}); \,
\overline{p}^j_{k,\text{next}} \triangleq \overline{p}^j_{k} + \xi^\star (\overline{p}^j_{k+1}-\overline{p}^j_{k}),
$$
where~$\xi^\star$ is the minimum~$\xi\in [0,\,1]$ such that:
\begin{equation*}
\begin{aligned}
	\textbf{Conv} \big( \{ \overline{p}^i_{k},\,  \overline{p}^j_{k},
	\, \xi \overline{p}^i_{k+1} +& (1-\xi) \overline{p}^i_{k},
	\, \xi \overline{p}^j_{k+1} + (1-\xi) \overline{p}^j_{k}\} \big)  \\
	& \cap \ \widehat{\mathcal{O}}(d_m) = \emptyset,
\end{aligned}
\end{equation*}
where $d_m>0$ is a small margin.
Similar to~\eqref{eq:eta}, the optimal~$\xi^\star$ can be computed by the bisection method.
Given $\overline{p}^i_{k,\text{next}}$ and $\overline{p}^j_{k,\text{next}}$,
the linear constraints that separate the free
space~$\textbf{Conv}(\mathcal{P}^\text{free})$
for $\mathcal{P}^\text{free} \triangleq \{ \overline{p}^i_k,\overline{p}^j_k,\overline{p}^i_{k,\text{next}},\overline{p}^j_{k,\text{next}} \}$
and  obstacle~$o \in \mathcal{I}_o$ are obtained as follows:
\begin{equation} \label{eq:safe-zone-SVM}
	\begin{aligned}
		&  \max_{ \{\widehat{a}_{k}^{ij} ,  \widehat{b}_{k}^{ij} , \delta\} }    \; \delta  ,   \\
		\textbf{s.t.} \; &   \widehat{a}_{k}^{ij,o} { }^{\mathrm{T}} \ p  \geq \delta + \widehat{b}_{k}^{ij,o}, \  \forall p \in \mathcal{P}^\text{free}; \\
		& \widehat{a}_{k}^{ij,o} {}^{\mathrm{T}} p \leq  \widehat{b}_{k}^{ij,o}, \ \forall p \in \mathcal{P}^{o} ;  \\
		&  \| \widehat{a}_{k}^{ij,o} \|_{2} = 1;
	\end{aligned}
\end{equation}
where~$\widehat{a}_{k}^{ij,o} \in \mathbb{R}^3$, $\widehat{b}_{k}^{ij,o} \in \mathbb{R}$ and $\delta \in \mathbb{R}$
are the optimization variables;
$\mathcal{P}^{o}$ are the vertices of obstacle $o \in \mathcal{I}_o$ after being inflated by~$d_m$.
Since~$\textbf{Conv} (\mathcal{P}^\text{free}) \cap \widehat{\mathcal{O}}(d_m) = \emptyset$,
it follows from the separating hyperplane theorem that $\delta>0$.
Thus, the optimization in~\eqref{eq:safe-zone-SVM} can be solved
as a quadratic program similar to \cite{Chen2022-2}.
Given these coefficients~$\widehat{a}_{k}^{ij,o}$ and $\widehat{b}_{k}^{ij,o}$,
the LOS constraint for agents $i$ and $j$ is restated as:
\begin{equation} \label{eq:safe-zone-constraint}
	{ \widehat{a}_k^{ij,o} } {}^\mathrm{T} p_{k}^{i} \geq  \widehat{b}_{k}^{ij,o}, \, \forall j \in \mathcal{N}^i, \, \forall k \in \mathcal{K},
\end{equation}
which is the separating hyperplane for obstacle~$o \in \mathcal{I}_o$. 

\subsection{Overall Algorithm}\label{sub:overall}
The objective of each searcher~$i \in \mathcal{N}_s$ is to minimize
the balanced cost,
which is the distance to the target and velocity:
\begin{equation} \label{eq:C^i-searcher}
	C^i \triangleq \frac{1}{2} Q_K \|p_{K}^{i}-p_{\text{tg}}^{i}\|_2^2 + \frac{1}{2} \sum_{k=1}^{K-1}Q_{k}\|p_{k+1}^{i}-p_{k}^{i}\|_2^2,
\end{equation}
where $Q_k>0$ are given positive definite parameters, $\forall k \in \mathcal{K}$.
Meanwhile, each connector~$i\in \mathcal{N}_c$ minimizes its distance to an intermediate point $\widetilde{p}^i_k$ determined by its neighbors, i.e.,
\begin{equation} \label{eq:C^i-connector}
	C^i \triangleq \frac{1}{2} \sum_{k=1}^{K} \|p_k^i-\widetilde{p}^i_k \|_2^2; \quad \text{and} \quad
    \widetilde{p}^i_k \triangleq \sum_{j \in \mathcal{N}^i} \alpha_j\, \overline{p}^j_k,
\end{equation}
where $\alpha_j\geq 0$ are the weighting constants
with $\sum_{j \in \mathcal{N}^i} \alpha_j\triangleq 1$,
which are often chosen to bias the child agents.
Thus, the overall trajectory optimization problem for agent~$i\in \mathcal{N}$
is summarized as follows:
\begin{equation} \label{convex-program}
	\begin{aligned}
		& \quad \quad \min _{\{\mathbf{u}^i, \mathbf{x}^i\} }\; C^{i} \\
		\mathbf{s.t.} \quad
		& x_{k}^{i}=\mathbf{A} x_{k-1}^{i}+\mathbf{B} u_{k-1}^{i},  \\
		& \| \Theta_a u^i_{k-1} \|_2  \le a_\text{max}, \, \| \Theta_v v_k^i \|_2  \le v_\text{max}, \forall k \in \mathcal{K},  \\
		& {a_{k}^{i j}}^{\mathrm{T}} p_{k}^{i} \geq b_{k}^{i j}, \, {a_{k}^{i,o}}^{\mathrm{T}} p_{k}^{i} \geq b_{k}^{i,o}, \, \forall j\neq i,\, \forall k \in \mathcal{K}, \\
		& \left \| p_{k}^{i} - c^{i j}_k \right \| _2  \leq d_c/2,
		\, \forall j \in \mathcal{N}^i, \, \forall k \in \mathcal{K}, \\
		& { \widehat{a}_k^{ij,o} } {}^\mathrm{T} p_{k}^{i} \geq  \widehat{b}_{k}^{ij,o}, \, \forall j \in \mathcal{N}^i, \, \forall k \in \mathcal{K}, \\
		& v^i_K=\mathbf{0}_3,\\
	\end{aligned}
\end{equation}
where~$\mathbf{u}^i$ and $\mathbf{x}^i$ are the stacked vectors
of $u^i_{k-1}$ and $x^i_k$, $\forall k \in \mathcal{K}$;
the last constraint is enforced to avoid an aggressive terminal state.
It is worth noting that the optimization~\eqref{convex-program} is a quadratically constrained
quadratic program (QCQP) as it only has a quadratic objective function
in addition to linear and quadratic constraints.
Thus, it can be fast resolved by off-the-shelf convex optimization solvers,
e.g.~\texttt{CVXOPT}~\cite{cvxopt}.
Moreover,
if a solution can not be found,
the predetermined trajectory $\overline{\mathcal{P}}$ is chosen as the fallback solution.

\begin{algorithm}[t!] \label{AL:IMPC-MS}
	\caption{Overall Algorithm}\label{algorithm}	\SetKwInOut{Input}{Input}
	\SetKwInOut{Output}{Output}
	\Input{$p_g$, $\mathcal{P}_{\text{tg}}$, $\mathcal{O}$}
	$\Gamma^i \leftarrow$ \texttt{ComNet}$(p_g,\mathcal{P}^i_{\text{tg}},\mathcal{O})$ \label{algline:formulated-network}\;
	$\overline{\mathcal{P}}^i(t_0) = \{ p_0^{i}, \cdots, p_0^{i} \}$ \label{algline:impc-init}\;
	\While{task not accomplished}
	{
		\For{$i \in \mathcal{N}$ \rm{concurrently} }{
			$\overline{\mathcal{P}}^j (t)$ $\leftarrow$ Receive  via communication\; \label{algline:commu}
			$\psi^{i,j}$ $\leftarrow$ Derive constraints~\eqref{eq:circle-constraint} and \eqref{eq:safe-zone-constraint}\;  \label{algline:sight-cons}
			$\zeta^i \leftarrow \{\psi^{i,j}, \forall j \in \mathcal{N}^i \}$\;
			\label{algline:get-nei-commu}
			$\zeta^i \leftarrow$ Add constraints \eqref{eq:inter-constraint} and \eqref{eq:obstacle-constraint}\;
			\label{algline:cons-collection}
			$\mathcal{P}^i(t)$ $\leftarrow$ Solve optimization~\eqref{convex-program} based on $\zeta^i$\; \label{algline:convex-programming}
		}
		$t \leftarrow t+h$\;
	}
\end{algorithm}

The complete planning algorithm is summarized in Alg.~\ref{AL:IMPC-MS}.
The design of communication topology is performed via Alg.~\ref{AL:span-tree}.
Afterwards, the predetermined trajectories of all agents are initialized.
In the main loop of execution,
the agents exchange their predetermined trajectories via communication in Line~\ref{algline:commu}.
At each time step,
the connectivity constraints including
$\psi^{ij} \triangleq \{ c^{ij}_k, \widehat{a}^{ij,o}_k,\widehat{b}^{ij,o}_k \}$
are calculated in Line~\ref{algline:sight-cons}
and exchanged among neighboring agents in Line~\ref{algline:get-nei-commu}.
In combination with the constraints for collision avoidance,
the overall constraints for~\eqref{convex-program} are stored as~$\zeta^i$
in Line~\ref{algline:cons-collection}.
Then, the optimization~\eqref{convex-program} is formulated and solved locally
by each agent~$i\in \mathcal{N}$ to obtain its planned trajectory
in Line~\ref{algline:convex-programming}.
It should be mentioned that the above procedure is applied to
all agents including the searchers and connectors.


\subsection{Property Analysis}\label{sub:property}
First, it can be shown that the proposed scheme ensures
the multi-agent system remains both safe and connected at all time steps,
as stated in the following theorem.

\begin{theorem}  \label{theo:connectivity-guarantee}
	If the agents are initially collision-free and connected,
	the system remains collision-free and the communication graph $G(t)$ remains connected at all times.
\end{theorem}

\begin{proof}
    First, it can be shown that the optimization in~\eqref{convex-program}
	is recursively feasible.
	At the initial time~$t_0$, the predetermined trajectory $\overline{\mathcal{P}}^i(t_0)= \left\{  p^{i}(t_0), \cdots, p^{i}(t_0) \right\}$
    satisfies all constraints in~\eqref{convex-program} as the agents are initially collision-free and the network is connected.
    Then, the optimization at time $t$ remains feasible since the constraints \eqref{eq:circle-constraint} 
    and \eqref{eq:safe-zone-constraint} are satisfied at $t$ if they are satisfied at $t-h$, as shown in Sec.~\ref{sub: connection-maintain}.
	
	Second, given a feasible solution at $t-h$ with control inputs $u^i_{k-1}(t-h)$ and state variables $x^i_k(t-h)$ for $k \in \mathcal{K}$,
	 the planned trajectories $p^i_k(t)=\overline{p}^i_k(t)$ are also feasible.
	For the constraint in~\eqref{eq:circle-constraint},
	 it is clear that $\overline{p}^i_k(t)$ lies within the constrained circle.
	 For the constraint in~\eqref{eq:safe-zone-constraint},
     $\overline{p}^i_k(t)$
     satisfies $ \widehat{a}_{k}^{ij} {}^\mathrm{T} \overline{p}^i_k(t)  \geq \delta + \widehat{b}_{k}^{ij}$,
     confirming its feasibility.
     The feasibility of other constraints has already been proven in \cite{Chen2022-2}.

	 Since both the initial and successive optimizations are feasible, the planned trajectories $\mathcal{P}^i$ derived from the optimization~\eqref{convex-program}
	  satisfy all constraints for all agents~$i \in \mathcal{N}$,
     which concludes the proof.
\end{proof}

\begin{remark}\label{remark:connectivity}
	The communication network~$G(t)$ is time-varying,
	with edges being added or removed as the agents move dynamically.
	However, the communication topology~$\mathcal{T}$ derived from Alg.~\ref{AL:span-tree}
	is always {contained} in~$G(t)$, meaning that the connectivity
	of~$\mathcal{T}$ is guaranteed at all time.
\hfill $\blacksquare$
\end{remark}

\begin{figure}[t]
    \centering
	\includegraphics[width=0.98\linewidth]{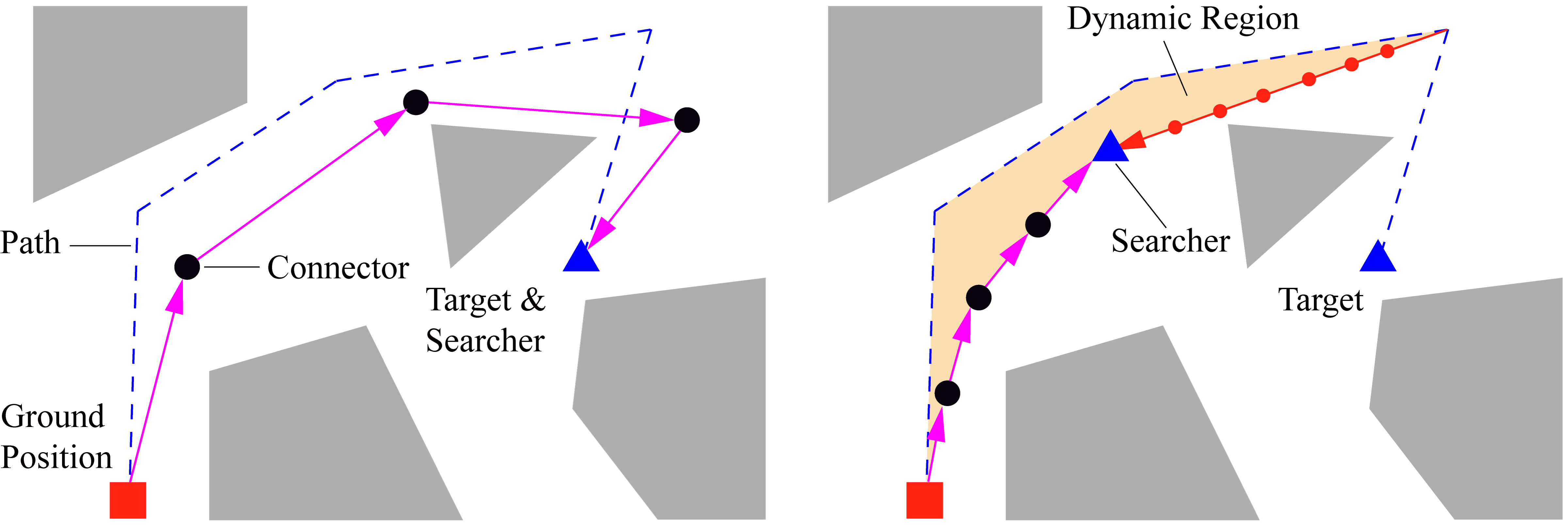}
    \vspace{-2mm}
	\caption{ \textbf{Left}: the chain (in solid line) from the ground station to a searcher
	has the same topological structure as the given path (in dashed line).
	\textbf{Right}: the area in orange shows the dynamic region. }
	\label{topology-structure}
    \vspace{-1mm}
\end{figure}

Another property concerns the evolution of this embedded tree in the workspace.
To begin with, the concept of {equivalent topologies} is defined as follows,
similar to~\cite{Gao2020}.
\begin{definition}
	Two polylines are \emph{topologically equivalent}
	if they share the same start and end points and can be transformed
	into each other without crossing any obstacle.\hfill $\blacksquare$
\end{definition}

\begin{theorem} \label{pro: topology}
	The polyline formed by the positions of agents from the ground station to a searcher is topologically
	equivalent to the path in the embedded tree~$\mathcal{T}$.
\end{theorem}

\begin{proof} (Sketch)
    Consider the dynamic region enclosed by the polyline,
    the searcher's trajectory, the line connecting $p^i_K$ and $p^i_\text{tg}$,
    and the path in the tree~$\mathcal{T}$,
    as shown in Fig.~\ref{topology-structure}.
    Initially, the region is empty with agents near the ground station.
    During motion, the polyline and trajectory are collision-free.
    The line segment from $p^i_K$ to $p^i_\text{tg}$ does not intersect obstacles,
    as both points are within a convex polyhedron
    due to constraint~\eqref{eq:obstacle-constraint}.
    For replanning, only the tractive point changes,
    and the loop formed by $p_\text{tg}(t-h)$, $\overline{p}^i_K$, $p_\text{tg}(t)$,
    and $\Gamma^i$ does not intersect obstacles.
    Thus, the dynamic region remains obstacle-free,
    confirming the topological equivalence.
\end{proof}

The above theorem shows that the polylines formed by all agents
remain consistent with the embedded tree~$\mathcal{T}$ at all times.
In other words, both searchers and connectors are deployed according to the designed tree topology. 

\section{Numerical Experiments} \label{sec:experiments}
This section presents the numerical simulations and hardware experiments to validate the proposed scheme.
The method is implemented in Python3 using \texttt{CVXOPT} \cite{cvxopt} for optimization.
All tests run on a computer with Intel Core i9 @3.2GHz of $10$ Cores and $32$GB memory.
The source code is available at https://github.com/CYDXYYJ/IMPC-MD.

\begin{figure}[t]
	\centering
	\includegraphics[width=0.95\linewidth]{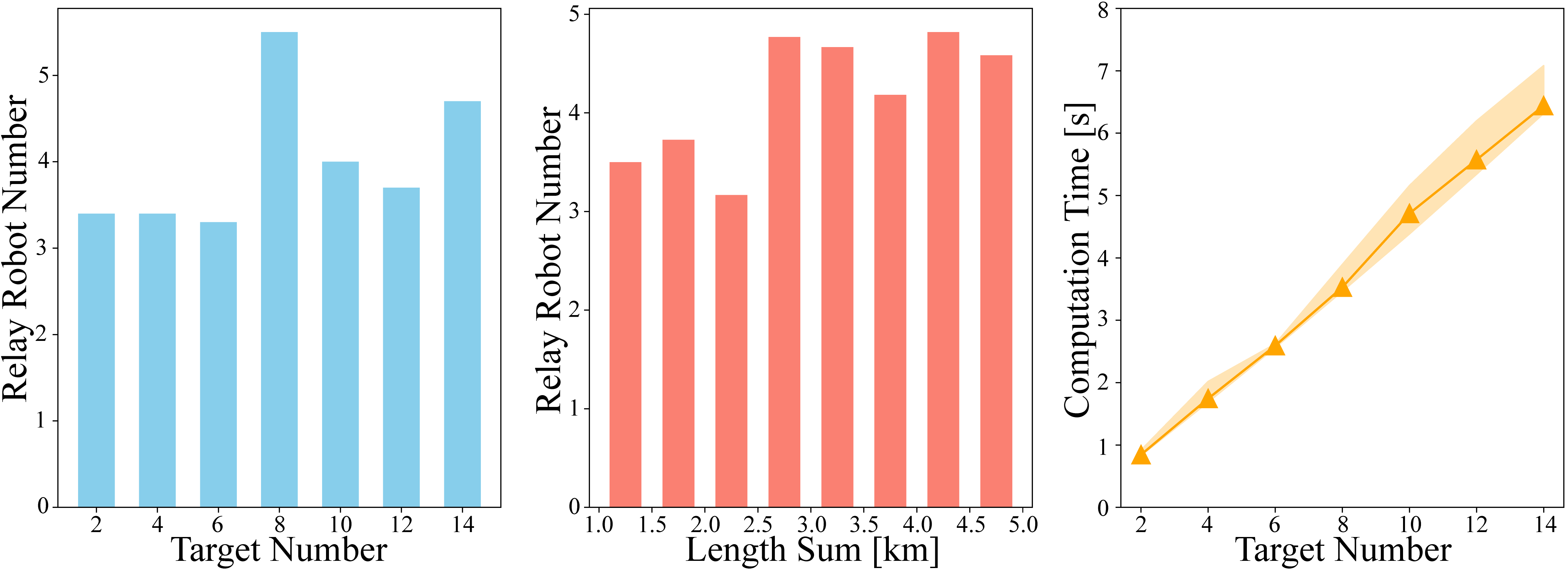}
    \vspace{-1mm}
	\caption{
		The communication topology via Alg.~\ref{AL:span-tree}
		given different numbers of targets,
		including the number of relay agents required (\textbf{Left}),
		the total length of the spanning tree (\textbf{Middle}),
		and the computation time (\textbf{Right}).
	}
	\label{fig:network-evaluation}
\end{figure}

\subsection{Numerical Simulations}
To mimic ariel vehicles,
the maximum velocity and acceleration of agents are set
to~$15 \rm{m/s}$ and $3 \rm{m/s^2}$, respectively.
The replanning interval for MPC is set to the sampling time~$h=0.5$s.
The safety radius of agent is set to $r_a=2$m.
The communication range and the margin of safe zone are chosen as
$d_c = 150$m and $d_m = 3$m, respectively.
The considered 3D workspace has a cubic volume from $(0\text{m},0\text{m},0\text{m})$ to $(500\text{m}, 500\text{m}, 100\text{m})$
with cluttered and convex-shaped obstacles, as shown in Fig.~\ref{fig:overall}.

\subsubsection{Design of Communication Network}
The proposed Alg.~\ref{AL:span-tree} is evaluated with different numbers of
target points within the obstacle-free space.
The algorithm computes the desired communication network by
\texttt{MiniEdgeRRT}$^\star$.
For each set of target points, the algorithm is run 15 times.
The results, including the number of relay nodes,
total length of the spanning tree, and computation time,
are summarized in Fig.~\ref{fig:network-evaluation}.
It is observed that the number of relay agents does not increase linearly
with the number of targets,
primarily due to the ``edge sharing" mechanism in Alg.~\ref{AL:span-tree}.
A similar pattern is observed for the total length of the spanning tree.
Finally, the computation time increases linearly with the number of targets,
as expected from the analysis in Sec.~\ref{sec:network-formulation}.

\begin{table} [t]
	\caption{
			Comparison with Baselines (Avg. over $10$ runs)
	}
    \vspace{-2mm}
	\label{table:comparison}
	\begin{tabular}{c|cccc|cccc}
		\toprule
		Metric & \multicolumn{4}{c}{$N_r$}  & \multicolumn{4}{c}{$N_h$} \\
		\cmidrule(r){2-5} \cmidrule(r){6-9}
		Targets & 2 & 4 & 6 & 8 & 2 & 4 & 6 & 8 \\
		\midrule
		{MST}             & {7.0} & {8.4} & {8.7} & {9.9}  & {10.7} & {27.7} & {38.2} & {61.6} \\
		VWMST & 7.1 & 8.9 & 9.7 & 10.8 & 8.5  & 19.1 & 26.3 & 37.5 \\
		DST             & 7.0 & 8.2 & 8.5 & 9.5  & 9.1  & 23.2 & 33.3 & 45.3 \\
		\textbf{Ours}            & 6.6 & 7.7 & 7.7 & 8.6  & 9.6  & 24.1 & 32.0 & 49.5 \\
		\bottomrule
	\end{tabular}
\end{table}

\begin{figure}[t]
	\centering
	\includegraphics[width=1.01\linewidth]{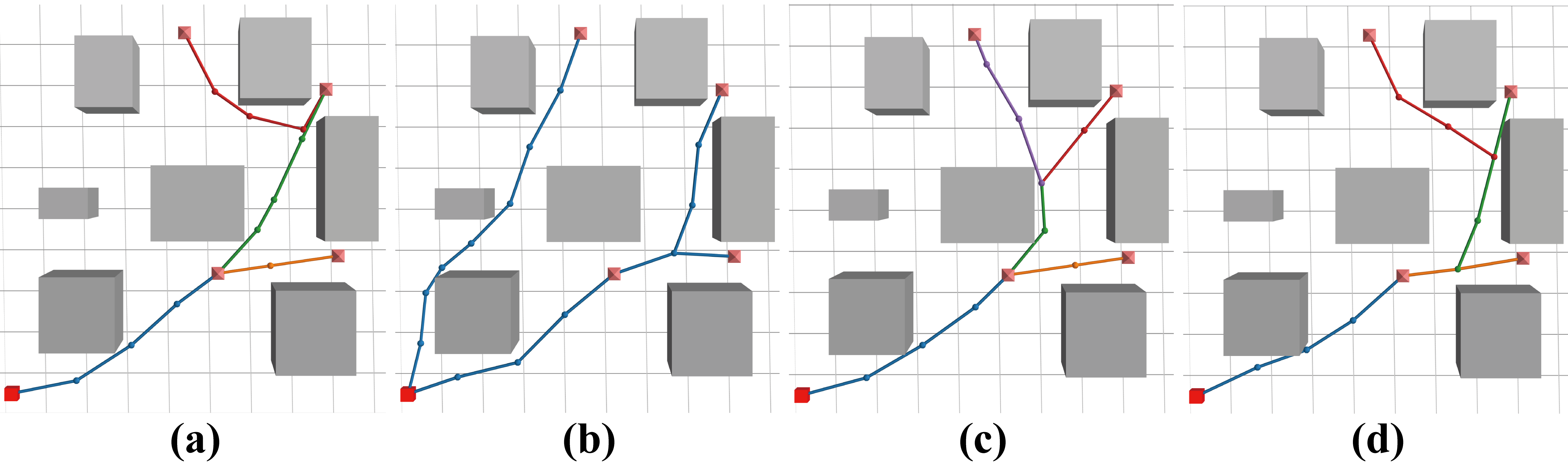}
    \vspace{-6mm}
	\caption{
		Comparison of the resulting communication topology
		for~$4$ target points,
		between \textbf{(a)} the MST-based method
		\textbf{(b)} the self-organizing method VWMST,
		\textbf{(c)} the distance-based spanning tree DST
		and \textbf{(d)}  our method.
		In addition to~$4$ searchers, $10$, $13$, $9$, $8$ connectors as relay agents
		are required by these methods, respectively.}
	\label{fig:network-compare}
\end{figure}

The proposed approach is compared with three baselines:
(i) Minimum Spanning Tree (MST),
(ii) Visibility Weighted Minimum Spanning Tree (VWMST) in \cite{Boldrer2021}, which selects the result with the fewest agents,
(iii) Distance-based Spanning Tree (DST), which greedily adds targets based on distance.
All methods are tested $10$ times,
comparing the average number of relay agents $N_r$ and hops from the ground station to targets $N_h$.
As shown in Table~\ref{table:comparison},
the our method requires only $8$ relay agents for $4$ targets on average.
Compared to VWMST, it uses less than $15\%$ of the relay agents for $4$, $6$, and $8$ targets.
VWMST, being decentralized,
often leads searchers to take direct paths to targets, resulting in fewer hops but not shared edges.
In comparison with MST and DST, our method requires slightly fewer relay agents, while MST has notably more hops.

\subsubsection{Collaborative Motion Planning}

\begin{figure}[t]
	\centering
	\includegraphics[width=0.98\linewidth]{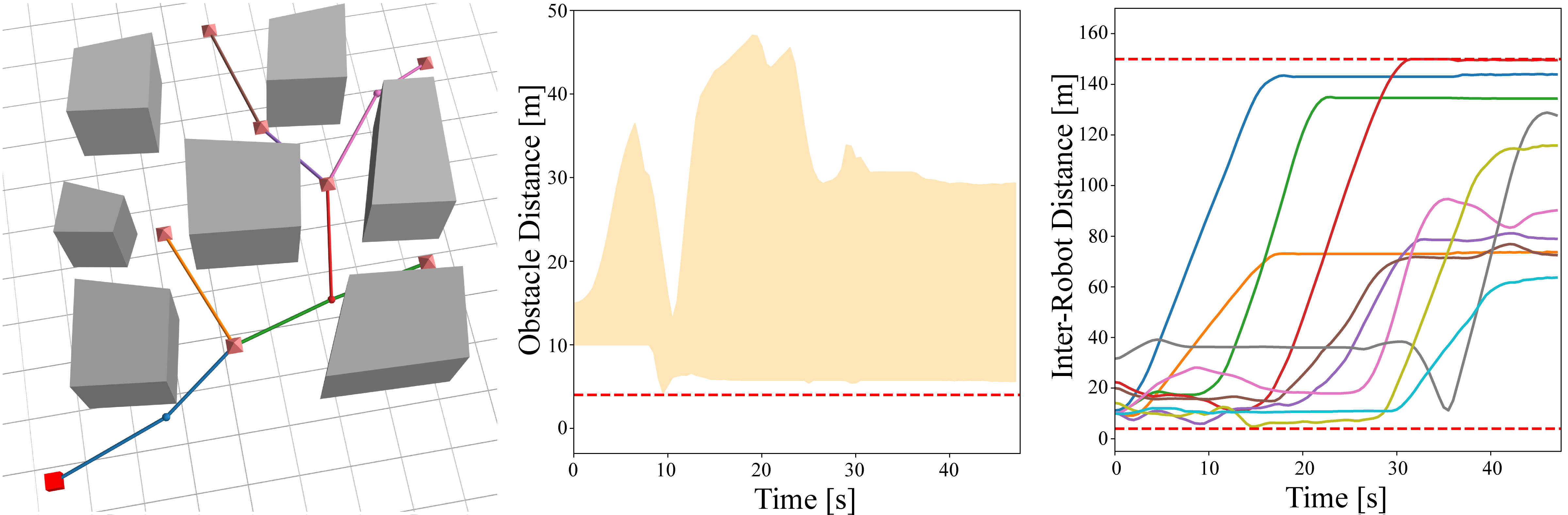}
    \vspace{-1mm}
	\caption{
		Illustration of the communication topology (\textbf{Left});
		the relative distance between the LOS of neighboring agents and the closet obstacle (\textbf{Middle}); and the relative distance between interconnected agents (\textbf{Right}).
	}
	\label{fig:simulation}
\end{figure}

\begin{figure}[t]
	\centering
	\includegraphics[width=0.95\linewidth]{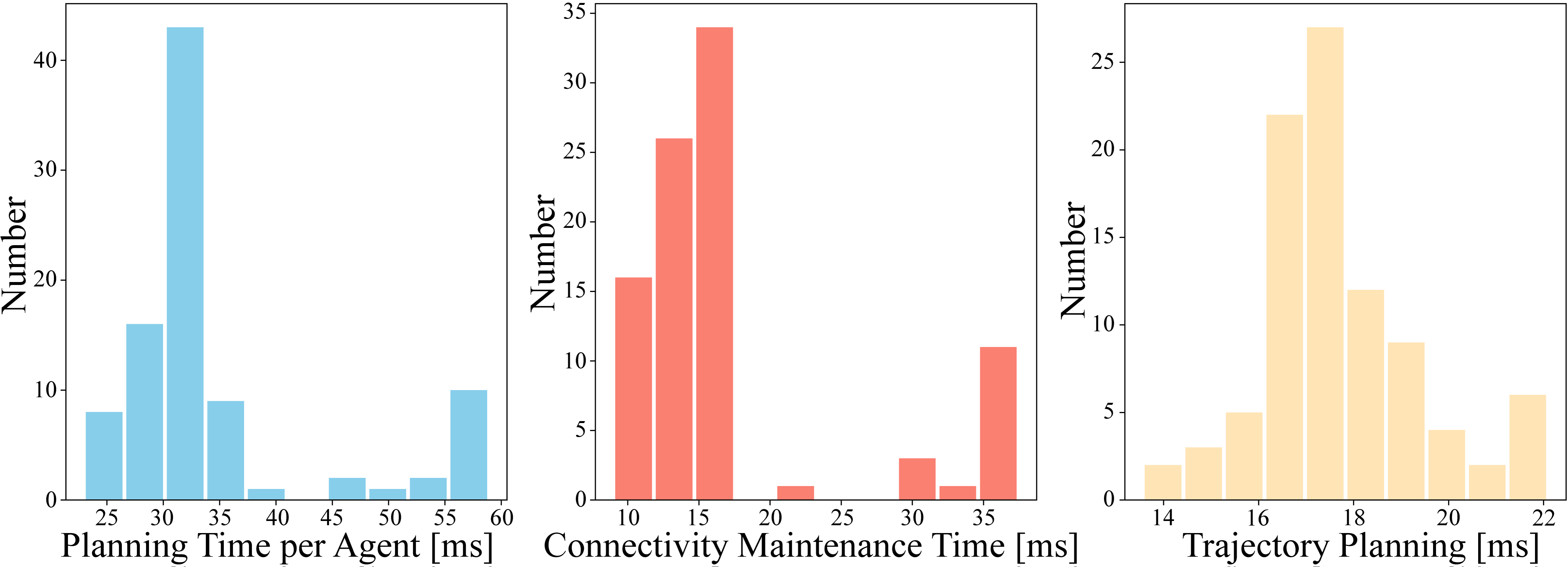}
    \vspace{-1mm}
	\caption{
		Distribution of the computation time of each iteration
		across agents (\textbf{Left}), which consists of:
		(i) the time to derive the constraints in~\eqref{eq:safe-zone-SVM} (\textbf{Middle});
		and (ii) the time to solve the final optimization in~\eqref{convex-program} (\textbf{Right}).
	}
	\label{fig:time-cost}
\end{figure}

The communication warning distance $d_w$ is set to $142$m,
and the margin of LOS is set to $3$m.
The weighting parameters $\alpha_j$ in \eqref{eq:C^i-connector} are set to $3$
for child neighbors and $1$ for parent neighbors.
The resulting trajectories under~$7$ targets within~$46$s are shown in Fig.~\ref{fig:overall} and~\ref{fig:simulation},
in addition to the relative distances between any two
agents and the minimum distance between their LOS to the obstacles.
The results show that the distance between any LOS and obstacles is always greater than $d_m = 3$m,
and the distance between neighboring agents stays below $d_c = 150$m.
Besides, the distance between any two agents stay positive at all times.
The average speed of all $7$ searchers reaches $80\%$ of the maximum speed,
ensuring efficient execution while maintaining safety and connectivity.
As shown in Fig.~\ref{fig:time-cost},
the average computation time per MPC iteration is $35$ms across $100$ iterations,
with $17.3$ms for constraints derivation in~\eqref{eq:safe-zone-SVM},
and $17.7$ms for solving the optimization in~\eqref{convex-program}.
The former however exhibits a larger variance across iterations due to
the varying proximal obstacle environment.


\begin{figure}[t!]
	\centering
	\includegraphics[width=0.98\linewidth]{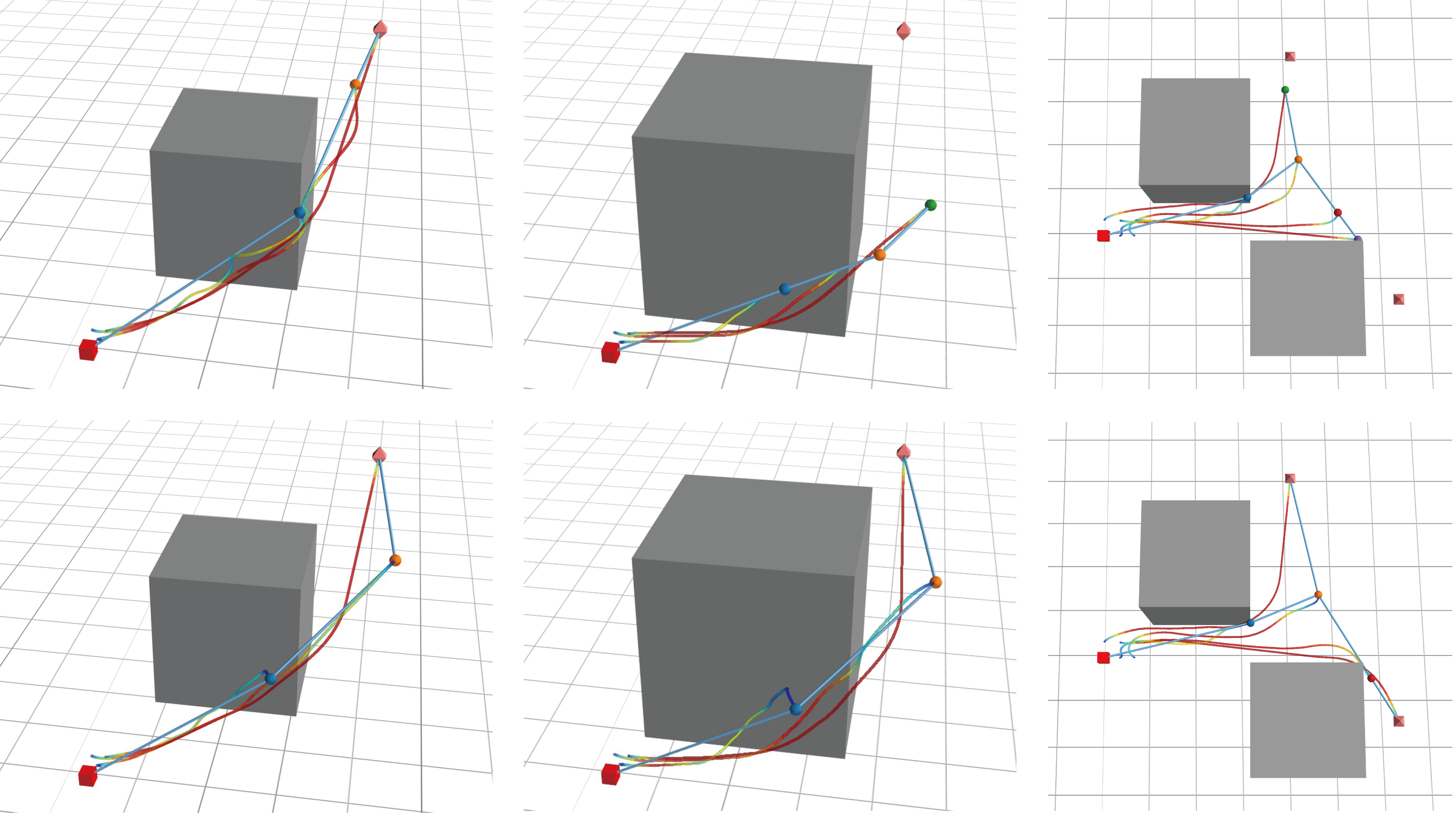}
    \vspace{-2mm}
	\caption{
		Comparisons of the resulting trajectories by the method in \cite{Caregnato2022}~(\textbf{Top})
		and our method~(\textbf{Bottom}) for different obstacle environments.
	}
	\label{fig:motion-compare}
\end{figure}

For comparison, the method in~\cite{Caregnato2022} is implemented via MILP.
The resulting trajectories in different workspaces are shown in Fig.~\ref{fig:motion-compare}.
For a cubic obstacle of width $100$m, the method in~\cite{Caregnato2022} takes over $857$ms to find a trajectory of length $850$m.
In comparison, our method takes only $65$ms and results in a trajectory of $765$m.
For a larger cubic obstacle of width $150$m, the method in~\cite{Caregnato2022} fails to complete the task,
while our method finishes within $30.0$s.
In another scenario with two obstacles requiring more relay points,
the method in~\cite{Caregnato2022} fails to generate a feasible solution
due to long computation times (over $10$s per replanning) and obstruction.
Our method, however, completes the task in $23.5$s with a trajectory length of $1291$m,
demonstrating its robustness in various workspaces.


\begin{figure}[t]
	\centering
	\includegraphics[width=1\linewidth]{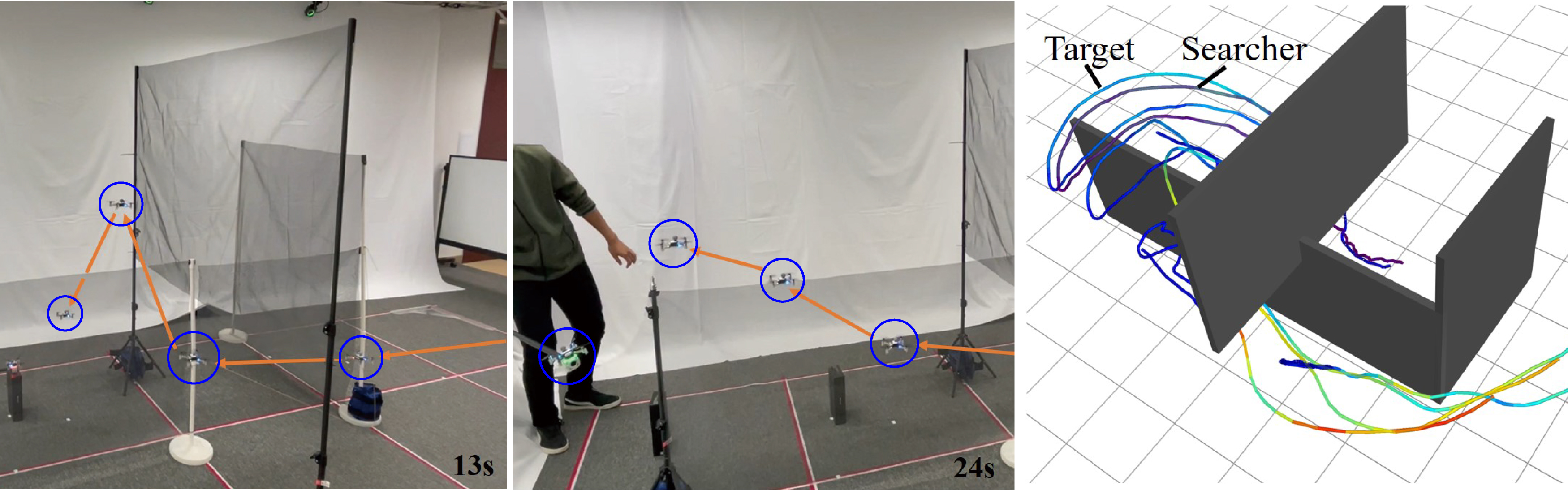}
    \vspace{-5mm}
	\caption{Illustration of how the network and the agent trajectories
		adapt online when the target is dynamically moving.}
	\label{fig:realfly}
\end{figure}

\begin{figure} [t!]
	\centering
	\includegraphics[width=1\linewidth]{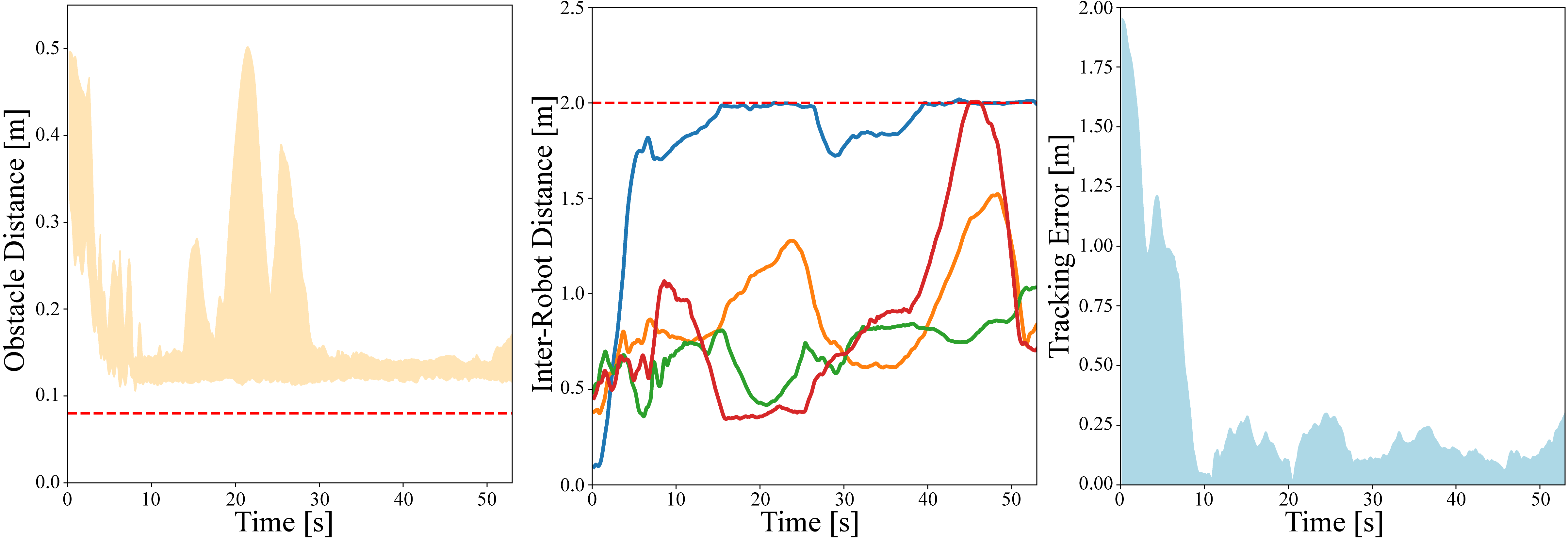}
    \vspace{-5mm}
	\caption{
		Numerical results from the experiment study:
		the minimum distance between the LOS of any neighboring agents and the obstacles (\textbf{Left});
		the distance between any pair of neighboring agents (\textbf{Middle});
		and the tracking error between the target and the assigned searcher (\textbf{Right}).
	}
	\label{fig:experiment}
\end{figure}

\subsection{Hardware Experiment}
Hardware experiments are conducted using Crazyswarm,
with motion tracked by OptiTrack and data transmitted to a workstation via Motive.
Due to the down-wash effect, each quadrotor is enclosed by a safety ellipsoid
with a diameter of $0.24$m in the $xy$ plane and $0.6$m along the $z$ axis.
The maximum velocity and acceleration are set to $1.0{\rm m/s}$ and $1.0{\rm m/s^2}$,
with a sampling time~$h=0.2s$ and a horizon length $K=8$.
The workspace includes known 3D walls.
Alg.~\ref{AL:IMPC-MS} runs on a central computer with multi-processing,
where each agent operates in an independent process.

In first experiment,
two targets are placed between the walls.
The communication topology
consists of $4$ connectors and $2$ searchers in an ``F''-shaped configuration, as shown in Fig.~\ref{fig:overall}.
The entire mission takes around~$10$s, with an average velocity of $0.4{\rm m/s}$ for all targets,
while ensuring collision avoidance and connectivity.
In the second experiment,
the target is \emph{moved manually} as shown in Fig.~\ref{fig:realfly}.
requiring the searcher to track it.
The whole fleet adapts to the movement, maintaining the motion constraints,
as shown in Fig.~\ref{fig:experiment}.
The tracking error stays below $0.25$m at all times,
demonstrating the system's robustness during online trajectory replanning. 
\section{Conclusion} \label{sec:conclusion}
This work presents a multi-UAV deployment method in obstacle-cluttered environments,
ensuring LOS connectivity and collision avoidance.
Different from most existing methods, the team size is designed,
and the distributed MPC guarantees feasibility and safety.
Future work will explore broader applications, such as exploration and surveillance. 

\bibliographystyle{IEEEtran}
\bibliography{contents/references}

\end{document}